\documentclass[twoside,11pt]{article}
\usepackage{times}
\usepackage{jmlr2e}

\usepackage{hyperref}
\usepackage{amsmath}
\usepackage{amssymb}
\usepackage{algorithm}
\usepackage{algorithmic}
\usepackage{tikz}
\usepackage{subfig}

\newcommand{\bE}{\,\mathbb{E}}
\newcommand{\reals}{\mathbb{R}}

\newcommand{\tr}{{\mathrm{tr}}}
\newcommand{\diag}{{\mathrm{diag}}}

\newcommand{\Var}{\mathrm{Var}}

\newcommand{\cF}{{\mathcal F}}

\newcommand{\cX}{\mathcal{X}}

\newcommand{\cU}{\mathcal{U}}
\newcommand{\cB}{\mathcal{B}}

\newcommand{\cA}{\mathcal{A}}
\newcommand{\cH}{\mathcal{H}}

\newcommand{\cN}{\mathcal{N}}

\newcommand{\todo}[1]{{\color{red}\bf TODO: #1}}

\makeatletter
\def\moverlay{\mathpalette\mov@rlay}
\def\mov@rlay#1#2{\leavevmode\vtop{%
   \baselineskip\z@skip \lineskiplimit-\maxdimen
   \ialign{\hfil$\m@th#1##$\hfil\cr#2\crcr}}}
\newcommand{\charfusion}[3][\mathord]{
    #1{\ifx#1\mathop\vphantom{#2}\fi
        \mathpalette\mov@rlay{#2\cr#3}
      }
    \ifx#1\mathop\expandafter\displaylimits\fi}
\makeatother

\newcommand{\conv}{\textrm{conv}}

\newcommand{\symps}[1]{\mathbb{S}^{#1}_{+}}

\newcommand{\inner}[1]{\langle #1 \rangle}
\newcommand{\finner}[1]{\langle #1 \rangle}
\newcommand{\fnorm}[1]{\| #1\|_F}
\newcommand{\snorm}[1]{\| #1\|_{\mathrm{sp}}}
\newcommand{\tnorm}[1]{\| #1\|_{\mathrm{tr}}}

\newcommand{\nnz}{\mathrm{nnz}}

\DeclareMathOperator*{\argmin}{argmin} 

\renewcommand{\eqref}[1]{Equation~(\ref{#1})}

\newcommand{\figref}[1]{Figure~\ref{#1}}
\newcommand{\secref}[1]{Section~\ref{#1}}

\newcommand{\appref}[1]{Appendix~\ref{#1}}
\newcommand{\thmref}[1]{Theorem~\ref{#1}}
\newcommand{\lemref}[1]{Lemma~\ref{#1}}

\newcommand{\corref}[1]{Corollary~\ref{#1}}

\newcommand{\algref}[1]{Algorithm~\ref{#1}}

\newcommand{\amsample}{m_a(d,k,r,\epsilon)}
\newcommand{\pmsample}{m_p(d,k,r,\epsilon)}

\newcommand{\err}{\textrm{err}}
\newcommand{\proj}{{\mathcal{P}}_k ^d}

\newcommand{\projConv}{{\mathcal{C}} _k ^d}

\newcommand{\sort}{\mathrm{sort}}

\jmlrheading{17}{2016}{1-21}{10/14; Revised 11/15}{4/16}{Alon Gonen, Dan Rosenbaum, Yonina C. Eldar and Shai Shalev-Shwartz}

\ShortHeadings{Subspace Learning with Partial Information}{Gonen, Rosenbaum, Eldar and Shalev-Shwartz}
\firstpageno{1}

\begin{document}
\newlength\figureheight
\newlength\figurewidth

\title{Subspace Learning with Partial Information}

\author{\name Alon Gonen \email alongnn@cs.huji.ac.il \\
       \addr School of Computer Science and Engineering\\ The
        Hebrew University\\ Jerusalem, Israel
       \AND
       \name Dan Rosenbaum \email danrsm@cs.huji.ac.il \\
       \addr  School of Computer Science and Engineering\\ The
        Hebrew University\\ Jerusalem, Israel
        \AND
        Yonina C. Eldar \email yonina@ee.technion.ac.il \\
        \addr Department of Electrical Engineering\\
        Technion, Israel Institute of Technology\\
        Haifa, Israel
        \AND
        Shai Shalev-Shwartz \email shais@cs.huji.ac.il\\
        \addr School of Computer Science and Engineering\\ The
        Hebrew University\\ Jerusalem, Israel
}

\editor{Kevin Murphy}

\maketitle

\begin{abstract}
The goal of subspace learning is to find a $k$-dimensional subspace
of $\mathbb{R}^d$, such that the expected squared distance between
instance vectors and the subspace is as small as possible. In this
paper we study subspace learning in a
\emph{partial information} setting, in which the learner can only
observe $r \le d$ attributes from each instance vector. We propose
several efficient algorithms for this task, and analyze their sample
complexity.
\end{abstract}


\section{Introduction}  \label{sec:intro}
Subspace learning is a dimensionality reduction technique in a
variety of applications such as face recognition~\citep{yang2004two}, image compression~\citep{du2007hyperspectral},
and document classification~\citep{papadimitriou1998latent}. Recently, there has been growing
interest in subspace learning from partially observed data (e.g.,~\citep{chen2013low,wang1998blind,chi2013petrels}). As motivation, consider the scenario of subspace learning from corrupted data. As discussed in \cite{chen2013low}, data corruption may cause some (or even most) of the attributes to be missing. Another typical scenario is subspace learning from multiple sources. Many applications (e.g., wireless sensor networks~\citep{chi2013petrels}) rely on data which is collected from multiple sources (e.g., sensors). When the data dimension is high, it may be impossible or prohibitively expensive to collect every data entry from every source. Note that in the first scenario, we have no control over which attributes are missing, while in the second one a learner may actively choose which attributes to observe.

The subspace learning problem is formally defined as follows. Let $\cX$
be a subset of the Euclidean unit ball in $\reals^d$, and let $P$ be some
unknown distribution over $\cX$. Our goal is to find a rank-$k$
projection matrix $\Pi \in \reals^{d\times d}$ such that the expected
squared distance, $ \bE_{x \sim P}[ \|x - \Pi x\|_2^2]$, is as small
as possible. 

When $\cX$ is a finite set and $P$ is the uniform distribution over
$\cX$, the optimal solution to the subspace learning problem is given by the
Principal Component Analysis (PCA) algorithm, which returns the projection matrix that corresponds to the $k$ leading eigenvectors of the matrix $\frac{1}{|\cX|} \sum_{x \in \cX} xx^\top$. 
In the more general stochastic optimization setting of subspace
learning, $\cX$ is not restricted to be a finite set, 
$P$ is an arbitrary distribution over $\cX$, and the information given to the learner has the form of an i.i.d. training sequence  $(x_1,\ldots, x_m) \sim P^m$.

In the usual full information setting of subspace learning, the learner
has access to all attributes of the sampled vectors. In this paper, we
consider subspace learning in a partial information setting, in which
only a subset of indices from each vector can be observed. Inspired
by the two applications presented above, we study two variants of this problem, which will be named the \emph{passive} setting, and
the \emph{active} setting, respectively. In the
passive setting, we assume that each attribute is observed with
probability $p=r/d$ ($r \le d$). Therefore, the expected number of observed
attributes from each vector is $r$. In the active setting, the learner
can choose (possibly at random) $r$ attributes to be revealed. The \emph{sample complexity} of a subspace learning algorithm is
defined as the number of samples that are needed by the algorithm in order to find a projection matrix $\Pi$ with expected squared distance, $ \bE_{x \sim P}[ \|x - \Pi x\|_2^2]$, of at most $\epsilon$ more than the optimal expected squared distance. The sample complexity for each of the models is defined as the minimal sample complexity attained by any algorithm. In this paper we propose efficient algorithms for both settings and analyze their sample complexity. We also provide several
lower bounds on the sample complexity that can be attained by any algorithm.

\subsection{Our Contribution}  \label{sec:main}
Our first observation is that subspace learning (in both the passive
and active models) using $r=1$ attributes
is impossible. Consider the task of learning a $1$-dimensional
subspace (a line) in $\reals^2$ (see \figref{fig:impossibility}). Let
$u_1 = (\alpha,\alpha)~,~u_2 =  (-\alpha,\alpha)$. Denote by $P_1$ the uniform distribution over
$\{u_1,-u_1\}$, and by $P_2$ be the uniform distribution over
$\{u_2,-u_2\}$. Suppose that the actual distribution $P$ is chosen uniformly at random from $\{P_1,P_2\}$. The task of subspace learning in this case amounts to
distinguishing between $P_1$ and $P_2$. However, since each single coordinate is distributed
uniformly over $\{-\alpha,\alpha\}$, the learner does not obtain any
information from a single observation, and hence cannot identify the right subspace.
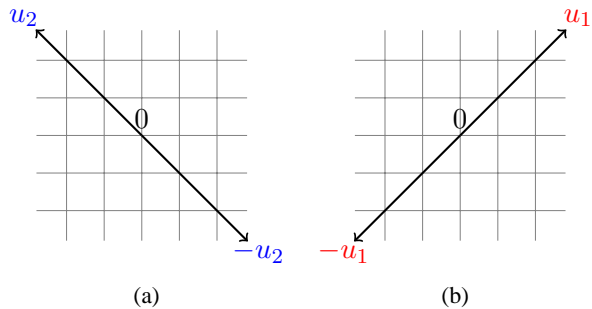
\begin{figure}[t]
\centering
\subfloat[]{
\begin{tikzpicture}[scale=1]
\draw[step=.5,help lines] (-1.4,-1.4) grid (1.4,1.4); 
\draw[thick][->] (0,0) -- (-1.41,1.41);
\draw[blue] (-1.57,1.57) node{$u_2$};
\draw[thick][->] (0,0) -- (1.41,-1.41);
\draw[blue] (1.55,-1.55) node{$-u_2$};
\draw[black] (0,0.22) node{$0$};
\end{tikzpicture}
}
\subfloat[]{
\begin{tikzpicture}[scale=1]
\draw[step=.5,help lines] (-1.4,-1.4) grid (1.4,1.4); 
\draw[thick][->] (0,0) -- (1.41,1.41);
\draw[red] (1.57,1.57) node{$u_1$};
\draw[thick][->] (0,0) -- (-1.41,-1.41);
\draw[red] (-1.55,-1.55) node{$-u_1$};
\draw[black] (0,0.22) node{$0$};
\end{tikzpicture}
}
\caption{Impossibility of subspace learning using $r=1$ observed attributes. The distribution of the observed attribute is identical for (a) and for (b), therefore they are indistinguishable.}
\label{fig:impossibility}
\end{figure}

For the case $2 \le r \le d$ we propose two efficient algorithms,
named Partially Observed PCA (\secref{sec:POPCA})
and Matrix Bandit Exponentiated Gradient (\secref{sec:MBEG}), which
are designed for the passive and active models, respectively. 

Partially Observed PCA (POPCA) is an extension of the PCA algorithm. Denote the
population covariance matrix $\bE[xx^\top]$ by $C$. The optimal
projection matrix, denoted $\Pi^\star$, is the projection onto the
subspace that is spanned by the $k$ leading eigenvectors of
$C$. An ERM (empirical risk minimizer, i.e., an algorithm that
minimizes the empirical loss) for the full information
setting is given by PCA which returns the projection
matrix corresponding to the $k$ leading eigenvectors of the matrix
$m^{-1}\sum_{i=1}^m x_i x_i^\top$ (whose expected value is
$C$). Similarly, the POPCA algorithm uses the random observations in
order to construct an estimate $\hat{C}$ of $C$, and approximates $\Pi^\star$ using the
projection matrix onto the $k$ leading eigenvectors of
$\hat{C}$. We analyze the sample complexity of this algorithm,
showing (see \corref{cor:POPCA}) that it is bounded from above by $(d/r)^2 \frac{k}{\epsilon^2}$. 
 
In the full information setting, the sample complexity is
known to be $O(k/\epsilon^2)$ (This bounds coincides with our bound
when $r=d$). Hence, the (multiplicative) price of
partial information, that is, how many more examples we need in order to
compensate for the lack of full information on each individual
example, is $O( (d/r)^2)$. It is interesting to understand whether a lower
price can be achieved. In \thmref{thm:lowerUniform} we prove that the sample complexity of every
algorithm in the passive model is $\Omega \left( (d/r)^2
\frac{k}{\epsilon^2} \right)$. The optimality of
POPCA in the passive mode is thus established. Another appealing property of POPCA is that, in terms of computational complexity, the challenge of partial information does not incur any additional cost; While the sample complexity grows as $r$ decreases, the runtime per iteration decreases by the same order.

Next, we investigate the active model and ask whether the price of
partial information can be reduced due to the ability of the learner
to actively choose the observed attributes. 
Intuitively, one may hope that the learning process would reveal
some useful information that can be utilized while choosing which coordinates to observe.
Our second algorithm, called Matrix Bandit Exponentiated Gradient (MBEG), exploits the
active setting by maintaining a ``weight matrix'' which is updated with every
observation, and induces a non-uniform attribute sampling
distribution. For MBEG, we derive an upper bound of $\max\left\{ 8k  \cdot \frac{d+r}{r} \cdot  \frac{k}{\epsilon^2}~,~
  \frac{d^4 r }{2(d+r)} \right\} \cdot \log(d)$ (see \thmref{thm:MBEG}) on the sample complexity. We note that if
$\epsilon$ is small enough, then the right term in the bound becomes
irrelevant, and thus a linear dependence on $d/r$ is obtained (for a
detailed comparison, see \secref{sec:comparison}). The (almost trivial)
fact that every subspace learner, even in the active model, must have a sample
complexity that grows linearly with $d/r$ is proved in
\appref{app:linearLower}. Hence, the dependence of MBEG on
$d/r$ is optimal, in the regime where $\epsilon$ is small.

The results in the active model immediately lead to the following question: Can we attain a linear price for partial information in the active model, independently of the required accuracy? We discuss possible directions for tackling this question in \secref{sec:discussion}.

\subsection{Related Work} \label{sec:related}
In the full information setting, it has been shown by
\cite{shawe2005eigenspectrum} and \cite{blanchard2007statistical}
that the optimal sample complexity of subspace learning is at most
$O(k/\epsilon^2)$, and this upper bound is achievable by applying
PCA on i.i.d. samples according to the distribution $P$. A similar
result is obtained by applying the Stochastic Gradient Descent
algorithm~\citep{arora2013stochastic}.

Subspace learning in the partial information setting has been studied
in \cite{chi2013petrels}, where an algorithm named PETRELS is proposed. However, no formal guarantees are derived for this
method. The setting in \cite{mitliagkas2014streaming} is similar
to our passive setting, but they assume that the distribution that
generates the instance vectors is Gaussian. 

A closely related problem to the task of subspace learning (in the passive setting) is the approximation of the covariance matrix using partially observed attributes. We discusses this relation in \secref{sec:POPCAanalysis}.

One possible way to tackle the challenge of subspace learning with partial
information is based on the matrix completion method. For example, we
may think of the partially observed examples as a data matrix with
unobserved entries. Then, one could first fill in the missing entries
using a matrix completion technique (e.g., as described in
\cite{candes2009exact}), and then apply PCA to the data matrix. We note that
this approach may work\footnote{Under some additional assumptions on the data such as
the incoherence assumption made in \cite{candes2009exact}} provided that the average number of observed attributes per example is sufficiently large. More precisely, according to the main result of \cite
{candes2009exact}, the number of observed attributes per example (column of the data matrix) should scale with the rank of the data matrix (which may be large as $d$). In contrast, our focus in this paper is on methods that work even when the number of observed attributes per example is much smaller (two attributes per instance suffice).

The active setting resembles the setting of the multi-armed bandit problem~\citep{auer2002nonstochastic}, in which the learner obtains limited feedback at each time, namely, it receives only the reward of the chosen arm. The challenge of learning linear predictors in $\reals^d$ with partially observed attributes (e.g., as in \cite{cesa2011efficient}) may be seen as an extension of this problem.  One of the most significant challenges in
this work is to adapt the technique used in the vector setting
(e.g., those employed by the Exp3 algorithm of \cite{auer2002nonstochastic}) to the corresponding matrix setting. We already observed one difference: While a
single arm suffices in the vector case, subspace
learning with $r=1$ attributes is impossible.

Our MBEG algorithm can be seen as an extension of the Online PCA
algorithm of \cite{warmuth2008randomized} (see also
\cite{nie2013online}) to the partial 
information setting. Similarly to their approach, MBEG maintains a weight matrix which induces a probability distribution over projection matrices. In MBEG, this weight matrix also induces a distribution over which attributes to observe.

\section{The Passive Setting} \label{sec:passive}
We begin by investigating the passive setting. We start by
describing an algorithm for this case. We then analyze its sample
complexity, and discuss its implementation.
\subsection{Partially Observed Principal Component Analysis (POPCA)}  \label{sec:POPCA}
In this section we describe the POPCA
algorithm. We start by reviewing the definition of the loss function and characterizing the minimizer of the loss. 

Denote the set of projection matrices from $\reals^d$ onto
$\reals^k$ by $\proj$. Since $\Pi^2 = \Pi$ for any $\Pi \in \proj$, the loss of a projection matrix $\Pi \in \proj$ can be expressed as
\begin{equation} \label{eq:loss}
L(\Pi) = \bE \|x-\Pi x\|_2^2 = 
\bE \left[\|x\|_2 ^2 - 2x^\top \Pi x + x^\top \Pi^\top \Pi x\right]
= \bE \left[\|x\|_2 ^2 - x^\top \Pi x \right] ~,
\end{equation}
Define the inner product of matrices $A,B$ by $\finner{A,B} =
\tr(A^\top B)$. We can further rewrite the loss as
\[
L(\Pi) = \bE \left[\|x\|_2 ^2 -
\finner{\Pi,xx^\top} \right] = \bE \left[\|x\|_2 ^2 \right] - \bE \left[
\finner{\Pi,xx^\top} \right] = \bE \left[\|x\|_2 ^2 \right] - 
\finner{\Pi,\bE [xx^\top]}~.
\]
Denote the covariance matrix $\bE[xx^\top]$ by $C$. Since $\|x\|_2^2$ does not depend on $\Pi$, the goal of subspace
learning is equivalent to finding a projection
matrix $\Pi$ such that
\begin{equation} \label{eq:obj1}
\finner{\Pi,-C} \le \min_{\Pi' \in \proj} \finner{\Pi',-C}
+ \epsilon ~.
\end{equation}
It is well-known that the rank-$k$ matrix which minimizes the expression $
\finner{\Pi',-C}$ is the matrix $\sum_{i=1}^k v_i v_i^\top$, where
$v_1,\ldots,v_k$ are the leading eigenvectors of $C$. The PCA approach
for subspace learning in the full information setting replaces $C$ with
$C(S) = \frac{1}{m} \sum_{i=1}^m x_i x_i^\top$, where $S =
(x_1,\ldots,x_m)$ is an i.i.d. training sequence drawn according to the
distribution $P$. Clearly, $\bE[C(S)] = C$. In our case, we will
construct an unbiased estimate of $C$ based on partially observed
examples, as detailed below. 

Consider an instance vector $x \sim P$ and let
$\hat{x}$ be the observed vector. According to our assumptions, for
each $i \in [d]$, the $i$-th coordinate of $\hat{x}$ satisfies
\begin{equation}  \label{eq:partial}
\hat{x}_i = \begin{cases} x_i & \textrm{w.p. } r/d \\ 0 &
  \textrm{w.p. } 1-r/d  ~. \end{cases}
\end{equation}
Denote $p=r/d$. Similarly to \cite{mitliagkas2014streaming}, we form the estimate
\begin{equation} \label{eq:oneEst}
\hat{C}_{\hat{x}} = \frac{1}{p^2} \hat{x} \hat{x}^\top
+\left(\frac{1}{p}-\frac{1}{p^2}\right) \diag(\hat{x}\hat{x}^\top)~.
\end{equation}
Indeed, it is easy to verify that $\hat{C}_{\hat{x}}$ forms an unbiased estimate of $C$:
\begin{align*}
\bE \left [\left(\frac{1}{p^2} \hat{x} \hat{x}^\top \right)_{i,j} \right] = \begin{cases}
\frac{1}{p} \bE [x_i^2] & i = j \\
\bE[x_ix_j]& i \neq j ~,
\end{cases} 
\end{align*}
and
\begin{align*}
\bE \left [\left(\frac{1}{p}-\frac{1}{p^2}\right) \diag(\hat{x}\hat{x}^\top)_{i,j} \right] = \begin{cases}
  \bE[x_i^2] - \frac{1}{p} \bE[x_i^2]& i = j \\ 0 & i \neq  j~. \end{cases}  
\end{align*}
Summing the corresponding entries, we see that $\bE[(\hat{C}_{\hat{x}})_{i,j}]
= \bE_{x \sim P} [x_ix_j] =  C_{i,j}$. Therefore, $\bE[\hat{C}_{\hat{x}}] =
C$. 

The POPCA algorithm (see \algref{alg:POPCA}) assumes access to $m$
i.i.d. vectors sampled according to $P$. For each instance vector,  it forms an estimate according to \ref{eq:oneEst}. By
averaging these $m$ estimates we obtain $\hat{C}$. The returned projection corresponds to the $k$ leading eigenvectors of $\hat{C}$.
\begin{algorithm}
\begin{algorithmic}
\caption{POPCA}
\label{alg:POPCA}
\STATE \textbf{Input: } $r,k \le d$ 
\STATE $\hat{C}=0 \in \reals^{d \times d}$
\FOR {$i=1$ \TO $m$}
\STATE Let $x_i \sim P$ and let $\hat{x}_i$ be the observed vector
\STATE $\hat{C}_{\hat{x}_i} = \frac{d^2}{r^2} \hat{x}_i \hat{x}_i^\top
+\left(\frac{d}{r}-\frac{d^2}{r^2}\right) \diag(\hat{x}_i\hat{x}_i^\top)$ 
\STATE $\hat{C}= \hat{C} + \frac{1}{m} \hat{C}_{\hat{x}_i}$
\ENDFOR
\STATE Compute the eigendecomposition $\hat{C}= \sum_{j=1} ^d \lambda_j v_j
v_j^\top$
\STATE Assuming $\lambda_1 \ge \ldots \ge \lambda_d$, return $\hat{\Pi} =
\sum_{j=1} ^k v_j v_j^\top$
\end{algorithmic}
\end{algorithm}

\subsection{Analysis of POPCA} \label{sec:POPCAanalysis}
The following lemma relates the success of \algref{alg:POPCA} to the quality of the estimation $\hat{C}$ of the covariance matrix $C$.
\begin{lemma} \label{lem:popcaReduceCov}
Suppose that the final estimate $\hat{C}$ of POPCA satisfies 
\begin{equation} \label{eq:popcaReduceCov}
\bE [\fnorm{C-\hat{C}}] \le \epsilon/\sqrt{k}~.
\end{equation}
Then, the resulting projection matrix $\hat{\Pi}$ satisfies the desired bound
\[
\bE [\inner{\hat{\Pi},-C}] \le \min_{\Pi' \in \proj} \inner{\Pi',-C}+\epsilon~.
\]
\end{lemma}
\begin{proof}
Using the Cauchy-Schwarz inequality, we get
\begin{align*} 
\bE [\sup_{\Pi \in \proj} \inner{\Pi,-C}-\inner{\Pi,-\hat{C}} ] &
\le  \bE [\sup_{\Pi \in \proj} \fnorm{\Pi} \fnorm{\hat{C}-C}]  \\&
\le \sup_{\Pi \in \proj} \fnorm{\Pi} \bE [\fnorm{\hat{C}-C}]  \\&
\le \sqrt{k} \cdot \epsilon/\sqrt{k}   \\& 
= \epsilon~.
\end{align*}
Consequently, since $\hat{\Pi}=\argmin_{\Pi \in \proj} \inner{\Pi,-\hat{C}}$, we have
\begin{align*}
\bE [\inner{\hat{\Pi},-C} -\inner{\Pi^\star,-C}] &= \bE
[\inner{\hat{\Pi},-C} -\inner{\Pi^\star,-\hat{C}}] \\&
\le \bE [\inner{\hat{\Pi},-\hat{C}} -\inner{\Pi^\star
,-\hat{C}}] + \epsilon\\&
\le \epsilon~,
\end{align*} 
which completes the proof.
\end{proof}
In view of \lemref{lem:popcaReduceCov}, obtaining upper bound on the sample complexity of POPCA reduces to analyzing the quality of the covariance estimation. A fairly vast body of literature exists on the latter task in the full-information scenario, and several results are known in the case of missing entries. For example, \cite{lounici2014high} considered a setting similar to our passive setting. Corollary 1 in \cite{lounici2014high} gives a bound on the Frobenius error that depends on the spectral decay of $C$.  We will derive a slightly different (worst-case) bound under the assumption that the instances are bounded in the Euclidean unit ball. A comparison between the two bounds is given \appref{app:lounici}.
\begin{lemma} \label{lem:covarianceConverge}
Let $\epsilon \in (0,1)$. 
If $m \ge \left \lceil (d/r)^2 \cdot \frac{k}{\epsilon^2}\right \rceil$, then 
\[
\bE [\fnorm{C-\hat{C}}] \le \epsilon/\sqrt{k}~.
\]
\end{lemma} 
\begin{proof}
Using Jensen's inequality, we have
\[
\bE[\|\hat{C}-C\|_F] = \bE[  ( \|\hat{C}-C\|_F^2)^{1/2} ] \le (\bE[
\|\hat{C}-C\|_F^2])^{1/2}~.
\]
Since the observations are i.i.d., 
\begin{align*}
\bE[ \|\hat{C}-C\|_F^2] &= 
\sum_{i,j} \Var[\hat{C}_{i,j}]  = \sum_{i,j}
\Var \left[(1/m) \sum_{q =1} ^m \hat{C}_{\hat{x}_q,i,j} \right] \\
&= \frac{1}{m}
\sum_{i,j} \Var[\hat{C}_{\hat{x}_1,i,j}] \le \frac{1}{m} \sum_{i,j} 
\bE[\hat{C}_{\hat{x}_1,i,j}^2] ~,
\end{align*}
where $\hat{C}_{\hat{x}_q,i,j}$ is the $[i,j]$-th entry of $\hat{C}_q$.
Denote $\hat{x} = \hat{x}_1$ and $x=x_1$ (subscript indices will now
correspond to the entries of these vectors). According to \ref{eq:oneEst},  
\[
\bE[\hat{C}_{\hat{x},i,j}^2|x]= \begin{cases} p^{-4} \bE[\hat{x}_i^2
  \hat{x}_j^2|x]  = p^{-2} x_i^2 x_j^2 & i \neq j \\ p^{-2}
  \bE[\hat{x}_i^4|x] = p^{-1} x_i^4  \le p^{-2} x_i^4 & i=j \end{cases}~.
\]
Therefore, since the $\ell_2$ norm of the instances is at most $1$, we have
\[
\sum_{i,j}  \bE[\hat{C}_{\hat{x}_1,i,j}^2|x] \le p^{-2} \sum_{i,j} x_i^2
x_j^2 = p^{-2} \|x\|^4 \le p^{-2} = \frac{d^2}{r^2}~,
\]
and $\sum_{i,j}  \bE[\hat{C}_{\hat{x}_1,i,j}^2] \le  \frac{d^2}{r^2}$.
Combining the above bounds, we obtain
\[
\bE[\|\hat{C}-C\|_F] \le \frac{1}{\sqrt{m}} \cdot \frac{d}{r}~.
\]
For $m \ge \left \lceil \frac{d^2}{r^2} \cdot \frac{k}{\epsilon^2} \right \rceil$, we arrive at the claimed bound.
\end{proof}
Let $\pmsample$ be the sample
complexity of subspace learning in the passive partial information
model, namely, how many examples are needed (for the optimal
learner) to guarantee that (\ref{eq:obj1}) holds. 
Based on \lemref{lem:covarianceConverge} and \lemref{lem:popcaReduceCov}, we now conclude the following bound on the sample
complexity. 
\begin{corollary} \label{cor:POPCA}
Using POPCA (\algref{alg:POPCA}), we have the following bound on the sample complexity for any integer $r \ge 2$:
\[
\pmsample \le \left \lceil  (d/r)^2  \cdot \frac{k}{\epsilon^2} \right \rceil~.
\]
\end{corollary} 

\subsection{Optimality of POPCA}
In this section we prove the following lower bound on the sample complexity of subspace learning with partial information in the passive model. 
\begin{theorem}  \label{thm:lowerUniform}
Assume that $k \le d/2$ and $\epsilon \in (0,1/128)$. The sample complexity in the passive model is at least $\Omega \left( (d/r)^2
\cdot \frac{k}{\epsilon^2} \right)$. Therefore, we have 
$$
\pmsample = \Theta \left((d/r)^2 \cdot \frac{k}{\epsilon^2} \right)~.
$$
\end{theorem}
Note that up to a constant factor, our lower bound coincides with the upper bound obtained by POPCA (\corref{cor:POPCA}). Therefore, \thmref{thm:lowerUniform} establishes the optimality of POPCA in the passive model. 

The proof of \thmref{thm:lowerUniform} is divided into two parts. First, in \thmref{thm:lowerFull} we prove that the sample complexity in the full-information setting is at least $\Omega(k/\epsilon^2)$. Then, we complete the proof of \thmref{thm:lowerUniform} by showing that the multiplicative price of partial information is at least $\Omega((d/r)^2)$.
\begin{theorem}  \label{thm:lowerFull}
Assume that $k \le d/2$ and let $\epsilon \in (0,1/128)$. The sample complexity of subspace learning
with full information is
bounded below by
$$
m(d,k,r=d,\epsilon) = \Omega(k/\epsilon^2)~.
$$ 
\end{theorem}
We now sketch the proof of \thmref{thm:lowerFull}. A detailed proof is provided in \appref{app:lowerFull}. 
\begin{proof} \textbf{(sketch)} \\
The idea is to reduce the problem of coin identification (see Section 5.2 in \cite{anthony2009neural}) to that of subspace learning.
Assume that $k \le d/2$ and $\epsilon \in (0,1/128)$. Let $\cU = \{u_j\}_{j=1}^{2k} \subseteq \reals^d$ be a set of $2k$ orthonormal vectors.
A  distribution over $\cU$  is defined as follows. First, we draw a
sequence $b=(b_1, \ldots,
b_k) \in \{-1,1\}^k$ uniformly at random. We associate the pair $\{u_i,
u_{i+k}\}$ with a Bernoulli random variable $B_i$ with parameter
$p_i = \frac{1 + b_i \alpha}{2}$, where $\alpha=16\epsilon$. To define a
distribution $P:=P_b$, we now describe the process of drawing an
instance $x \sim P$:
\begin{enumerate}
\item
An integer $i \in
[k]$ is chosen uniformly at random. 
\item
The $i$-th coin is flipped
(according to $p_i$). Denote the corresponding random variable by $Z$.
\item
If $Z = 1$, then $x$ is chosen uniformly at random from the set
$\{u_i,-u_i\}$. Otherwise ($Z = 0$), $x$ is chosen uniformly at random from the set
$\{u_{i+k},-u_{i+k}\}$.
\end{enumerate}
In \lemref{lem:subspaceToCoins} we show that a
successful subspace learner must identify the bias of ``most'' of the
coins. Thus, we can reduce $k$ independent tasks of coin
identification to the task of subspace learning. A well-known result in statistics \citep{anthony2009neural}[Lemma 5.1] tells us that
$\Omega(1/\alpha^2)$ samples are needed to identify a coin with
bias $\alpha$. Hence, each of the pairs must be observed
$\Omega(1/\epsilon^2)$ times.
\end{proof}
\begin{proof} \textbf{(of \thmref{thm:lowerUniform})}
Consider the construction presented in the proof (sketch) of \thmref{thm:lowerFull}. We next specify the set $\cU$ and prove that the price of
partial information in the passive model is
$\Omega(d^2/r^2)$. Consequently, this will conclude the proof of \thmref{thm:lowerUniform}. 

For every $i \in [k]$, let $u_i = \frac{\sqrt{2}}{2} (e_i+e_{i+k})$, $u_{i+k}
= \frac{\sqrt{2}}{2} (-e_i + e_{i+k})$. To specify a distribution $P$, fix a
vector $b=(b_1,\ldots,b_k) \in \{-1,1\}^{k}$. Consider now a single
interaction between the learner and the environment. A vector $x$ is
drawn according to $P$ as described in the proof (sketch) of \thmref{thm:lowerFull}. Let $i \in [k]$
denote the index of the coin which is associated with $x$. As we observed in \secref{sec:intro},
each of the coordinates $i$, $i+k$ is distributed uniformly over
$\{-\frac{\sqrt{2}}{2}, \frac{\sqrt{2}}{2}\}$. Hence, to obtain any
information, the learner must observe both of the coordinates $i$ and
$i+k$. The probability that both coordinates are observed is
at most $O(r^2/d^2)$. Hence, in expectation, (only) a single ``meaningful''
observation is obtained every $\Omega(d^2/r^2)$ iterations. Therefore, the
price of partial information is $\Omega(d^2/r^2)$.
\end{proof}

\subsection{Implementation of POPCA}
As we discussed above, when the data is fully visible, the sample
complexity of the PCA algorithm (which computes the $k$ leading eigenvectors of the empirical
covariance matrix, $m^{-1}\sum_{i=1} ^m x_i x_i^\top$, and returns the
corresponding projection matrix) is $m_f :=O(k/\epsilon^2)$. When the number
of samples, $m_f$, is larger than the dimension, a standard implementation
of this algorithm costs $O(m_fd^2)$. We next show that POPCA has a
similar runtime. 

We established above that POPCA requires a
training set of size $O((d/r)^2m_f)$. Consider a single iteration of
POPCA. Note that the construction of $\hat{C}_{\hat{x}_i}$ (see \ref{eq:oneEst}) costs $O(r^2)$. Therefore, it costs $O(m_fd^2)$ to obtain the average of the estimates, $\hat{C}$. The computation of the SVD of $\hat{C}$ costs $O(d^3)$. Therefore, when $m_f \ge d$, the overall
runtime of POPCA is indeed $O(m_fd^2)$.

\section{The Active Setting}  \label{sec:active}
We now consider the active model of subspace learning with
partial information. In particular, we will present and analyze the Matrix Bandit Exponentiated Gradient (MBEG)
algorithm. 

Before describing MBEG, it should be noticed that POPCA (along with its analysis) can be modified to fit the active model; An active version of POPCA simply selects the $r$ observed attributes uniformly at random (with replacement) and then proceeds similarly to POPCA\footnote{Note that here the attributes are no longer independent and therefore we should replace the weights in the definition of $\hat{C}_{\hat{x}}$.}. It is not hard to verify that the sample complexity of this algorithm is asymptotically equivalent to the sample complexity of POPCA. In particular, the price of partial information is quadratic in $d/r$. 

The main differences between MBEG
and POPCA are as follows:
\begin{enumerate}
\item
MBEG is designed for the active model. It employs a non-uniform
sampling method which gives higher
priority to ``stronger'' directions.
\item MBEG is an iterative algorithm which maintains a \emph{weight
matrix} that belongs to the convex hull of the set $\proj$ of
projection matrices. It can be thought of as a Bandit version of the extension of the Exponentiated Gradient (EG) algorithm  to matrices. The EG algorithm, and its extension to matrices are due to \cite{kivinen1997exponentiated}, and \cite{tsuda2005matrix}, respectively.
\end{enumerate}  

\subsection{ Matrix Bandit Exponentiated Gradient (MBEG)}  \label{sec:MBEG}

\subsubsection{Convexification}
In order to be able to apply the Matrix EG algorithm, we first need to formulate our
task as a convex optimization problem. Recall that our problem is 
equivalent to approximately minimizing the objective $\argmin_{\Pi \in \proj}
\finner{\Pi,-C}$, where $C = \bE[xx^\top]$. The objective is linear
and thus convex. We will replace the non-convex set $\proj$ with
its convex hull, $\projConv := \conv(\proj)$. Note that 
for every $W$ in $\projConv$, the gradient is given by $-C$. Therefore,
the EG procedure would start with some $W_1 \in
\projConv$, and, at iteration $i$, would update according to
\begin{align}  \label{eq:updateEG}
1.&~~~U_{i+1} =  \exp(\log(W_i) + \eta C)  \notag \\
2.&~~~W_{i+1} = \argmin_{W \in \projConv} D_R(W,U_{i+1}) ~,
\end{align}
where $D(R,U) = \tr(W \log W - W \log U- W + U)$ is the Bregman
divergence induced by the \emph{quantum entropy} regularizer, $R(W) = \tr(W
\log W - W)$. Additional details regarding this regualarizer can be
found in \cite{tsuda2005matrix} and \cite{warmuth2008randomized}. As in POPCA, 
the gradient $C$ is unknown but can be
estimated. We would like to
exploit the active setting, and therefore, we will
employ a non-uniform sampling method which relies on the current
weight matrix $W_t$ (see \secref{sec:nonUniform}). 

Next, we recall that the output of the algorithm
should be a projection matrix, while our algorithm maintains weight
matrices, which may not belong to the set $\proj$. Therefore, the
final step of the algorithm is to construct an element from $\proj$
that performs ``similarly'' to the average of the weight matrices
maintained during the run of the algorithm. For this purpose, we rely on the following lemma, due to
\cite{warmuth2008randomized}:
\begin{lemma} \label{lem:convChar}
Every matrix $W \in \projConv$ can be decomposed in
time $O(d^3)$ into a convex combination of at most $d$ elements from
$\proj$.  
\end{lemma}
For completeness, we recall the decomposition procedure of
\cite{warmuth2008randomized} in \appref{app:decomposition}. Getting
back to our algorithm, let $\hat{W} = \frac{1}{m} \sum_{i=1} ^m W_i$
be the average weight matrix, and denote by
$\hat{W} = \sum_{j=1} ^d \beta_j \Pi_j$ a decomposition of
$\hat{W}$ into a convex combination of elements from
$\proj$. The final step of MBEG sets the output matrix $\hat{\Pi}$ to be $\Pi_j$ with
probability $\beta_j$. This guarantees that the expected performance
of $\hat{\Pi}$ is the same as the performance of $\hat{W}$. 

\subsubsection{Prioritizing ``stronger'' directions}  \label{sec:nonUniform}
In this part we present the non-uniform sampling mechanism which is
employed by MBEG for attribute sampling. We first consider the case
$r=2$. Let $W : = W_i$ be a weight matrix obtained during the run of
MBEG. Recall\footnote{These properties clearly hold for
projection matrices, and thus hold for any convex combination of
projection matrices.} that $\sum_i
W_{i,i} = k$, and for every $i \in [d]$, $W_{i,i} \in
[0,1]$. Therefore, a natural distribution for attribute sampling is to choose each pair $(s,q) \in [d]^2$ 
with probability $p_{s,q} = \frac{W_{s,s}}{k} \cdot
\frac{W_{q,q}}{k}$. Unfortunately, we were not able to obtain a good bound
using this sampling technique. Instead, we pick attributes according to
\begin{equation} \label{eq:nonuniMBEG}
p_{s,q} =
(1-\alpha) \frac{W_{s,s} +W_{q,q} }{2dk} +  \frac{\alpha}{d^2}~,
\end{equation}
for some parameter $\alpha \in (0,1/2)$, which is tuned below. That is,
we mix a uniform distribution over $[d]^2$, with a
distribution which gives higher sampling probability to pairs for
which $W_{s,s},W_{q,q}$ are high, reflecting a bias toward sampling
from ``stronger'' directions. Mixing with the uniform
distribution guarantees that every pair has large enough probability
to be sampled, which will later help us ensure that we perform enough
``exploration''. 

Based on this probability distribution over pairs, we define an
unbiased estimate of the matrix $C$ by
\begin{equation} \label{eqn:hatCEG} 
\hat{C} =\frac{1}{2p_{s,q}}
  x_s x_q (E_{s,q}+E_{q,s})~,
\end{equation}
where $E_{s,q}$ is the all zeros matrix except $1$ in the $i,j$
coordinate. 

The extension of MBEG to any (even) $r>2$ is straightforward. We
simply pick $r/2$ independent estimates $\hat{C}_1,\ldots,
\hat{C}_{r/2}$, each of which is constructed as in the case $r=2$, and
set $\hat{C} = \frac{2}{r}
\sum_{j=1} ^{r/2} \hat{C}_j$. 

The algorithm is summarized in \algref{alg:MBEG}. As explained in \cite{tsuda2005matrix}, the projection step w.r.t. the Bregman divergence, $W_{i+1} = \argmin _{W \in \projConv} D_R (W,U_{i+1})$ can be performed in time $O(d^3)$. Overall, the running time per iteration is $O(d^3)$.

\begin{algorithm}
\begin{algorithmic}
\caption{Matrix Bandit Exponentiated Gradient}
\label{alg:MBEG}
\STATE \textbf{Input: } $r,k<d$ ($r \mod 2=0$)
\STATE $\eta=\sqrt{\frac{r\log (d) }{2m (d+r) }}$
\STATE $\alpha =  \eta d^2$ ~~(we assume that $m$ is large
enough so that $\alpha \le 1/2$)
\STATE $W_1 =  \frac{k}{d} I$
\FOR {$i=1$ \TO $m$}
\STATE denote $W=W_i$
\FOR {$j=1$ \TO $r/2$}
\STATE pick $(s,q) \in [d]^2$ with probability $p_{s,q} =
(1-\alpha) \frac{W_{s,s} +W_{q,q} }{2dk} + \frac{\alpha}{d^2}$
\STATE for $x \sim P$, let $(x_{s},x_{q})$ be the corresponding attributes
\STATE  $\hat{C}_{i,j} =
\frac{1}{2p_{s,q}}\,
  x_s x_q (E_{s,q}+E_{q,s})$ 
\ENDFOR
\STATE  $\hat{C}_i = \frac{2}{r} \sum_{j=1} ^{r/2} \hat{C}_{i,j}$
\STATE $U_{i+1} = \exp(\log(W)+ \eta \hat{C}_i)$
\STATE $W_{i+1} = \argmin _{W \in \projConv} D_R (W,U_{i+1})$
\ENDFOR
\STATE $\hat{W} = \frac{1}{m} \sum_{i=1} ^m W_i$
\STATE decompose $\hat{W}$ into $\hat{W} = \sum_{j=1} ^d \beta_j \Pi_j$ using \algref{alg:decomposition}
\STATE return $\hat{\Pi} = \Pi_j$ with probability $\beta_j$ 
\end{algorithmic}
\end{algorithm}

\subsubsection{Analysis of MBEG}
Let $\amsample$ be the sample
complexity of subspace learning in the active partial information
model. We denote the Frobenius norm, the spectral norm, and the trace norm
by $\fnorm{\cdot}$, $\snorm{\cdot}$, and $\tnorm{\cdot}$, respectively. Throughout this section we prove the following result.
\begin{theorem} \label{thm:MBEG}
Using MBEG (\algref{alg:MBEG}), we have the following bound on the sample complexity for any (even) integer $r \ge 2$:
\[
\amsample \le \max\left\{8k  \cdot \frac{d+r}{r} \cdot  \frac{k}{\epsilon^2}~,~
  \frac{2d^4 r }{d+r} \right\} \cdot \log(d)~.
\]
\end{theorem}
In order to prove \thmref{thm:MBEG}, we next apply the general analysis of Matrix EG~\citep{hazan2012near}[Theorem 13] to our case. 
\begin{theorem}    \label{thm:basicMEG}
Assume that the sequence $\hat{C}_1,
\ldots, \hat{C}_m$ obtained during the run of MBEG satisfies $\snorm{\hat{C}_i} \le 1/\eta$ for every $i \in [m]$. Then, for every $\Pi^\star \in \proj$,
\begin{equation}  \label{eq:basicMEG}
\sum_{i=1} ^m \finner{W_i-\Pi^\star ,-\hat{C}_i} \le \frac{k \log (d)}{\eta} +
\eta  \sum_{i=1} ^m \inner{W_i, \hat{C}_i ^2}~.
\end{equation}
\end{theorem}
The right-most term in \ref{eq:basicMEG} can be thought as the variance which is associated with the estimation process of MBEG. We now show that in contrast to POPCA, the variance
scales only linearly with $d/r$. 
\begin{lemma}  \label{lem:varMBEG}
For any matrix $W_i \in \projConv$ maintained by MBEG at time
$i$, denote the conditional expectation given $W_i$ by $\bE_i$. Then, 
\[
\bE_i \, \finner {W_i ,\hat{C}_i^2}  \le \frac{2k(d+r)}{r}~.
\]
\end{lemma}
We now prove the lemma while assuming that $r=2$. The extension
to any $r>2$ is detailed in \appref{app:mbegRg2}.
\begin{proof}
Let $W=W_i, \hat{C}=\hat{C}_i$. Then,
\begin{align*}
\bE_i \, \finner {W ,\hat{C}^2} &= \sum_{(s,q) \in [d]^2} p_{s,q} (W_{s,s}+W_{q,q})
\frac{1}{4p_{s,q}^2}  x_s^2 x_q^2  \\&
= \sum_{(s,q) \in [d]^2} 
\frac{(W_{s,s}+W_{q,q})}{4p_{s,q}}  x_s^2 x_q^2~.
\end{align*}
Since $\alpha \in (0,1/2)$, we have 
\[
\frac{W_{s,s}+W_{q,q}}{p_{s,q}} =
\frac{W_{s,s}+W_{q,q}}{(1-\alpha)\frac{W_{s,s} + W_{q,q}}{2dk} +
  \alpha/d^2} \le \frac{W_{s,s}+W_{q,q}}{(1-\alpha)\frac{W_{s,s} +
    W_{q,q}}{2dk}} \le 4dk~.
\]
Therefore, since $r=2$, we obtain
\begin{equation*} 
\bE_i \, \finner {W ,\hat{C}^2} \le dk \sum_{(s,q) \in [d]^2} 
 x_s^2 x_q^2 \le dk < \frac{2 k(d+r) }{r}~,
\end{equation*}
completing the proof of the lemma.
\end{proof}
\begin{proof} \textbf{(of \thmref{thm:MBEG})}
By definition of $\hat{C}_{i,j}$ we have that $\hat{C}_{i,j}^2 =
\hat{C}_{i,j} \hat{C}_{i,j}^\top$ is a diagonal matrix with all elements on
the diagonal equal to zero except the $(s,s)$ and $(q,q)$ elements,
which are both bounded above by $\frac{x_s^2x_q^2}{p_{s,q}^2}$. It follows that
\[
\snorm{\hat{C}_{i,j}} \le \frac{|x_sx_q|}{p_{s,q}} \le \frac{1}{p_{s,q}}
\le \frac{d^2}{\alpha} = \frac{1}{\eta} ~.
\]
Hence, 
\[
\snorm{\hat{C}_i} \le \frac{2}{r} \sum_{j=1} ^{r/2} \snorm{\hat{C}_{i,j}}
\le \frac{2}{r} \frac {r}{2} \frac{1}{\eta} = \frac{1}{\eta}~.
\]
Therefore, the conditions of \thmref{thm:basicMEG} hold. Taking expectation over \ref{eq:basicMEG}, we obtain
\begin{equation} \label{eqn:PlugHere}
\sum_{i=1} ^m \bE \, \finner{W_i-\Pi^\star
  ,-\hat{C}_i} \le \frac{k \log (d)}{\eta}   + \eta  \sum_{i=1}^m \, \bE \, \finner{W_i,
\hat{C}_i^2}~.
\end{equation}

Let $\bE_i$ denote the conditional expectation given $W_i$. Then,
$\bE_i [\hat{C}_i] = C$ and therefore, by the
law of total expectation, 
\begin{equation} \label{eqn:ToPlug2}
\bE ~\finner{W_i-\Pi^\star
  ,-\hat{C}_i}=\bE \finner{W_i-\Pi^\star ,-C} ~.
\end{equation}
Combining \ref{eqn:ToPlug2} with \lemref{lem:varMBEG}, and plugging
into \ref{eqn:PlugHere}, we obtain that
\[
\bE \sum_{i=1} ^m \finner{W_i-\Pi^\star ,-C} \le  \frac{k\log
    (d)}{\eta} + \eta m\frac{2k(d+r)}{r} ~.
\]
Dividing by $m$, denoting $\hat{W} =\frac{1}{m} \sum_{i=1} ^m  W_i$,
substituting $\eta=\sqrt{\frac{r\log (d) }{2(d+r) m }}$, 
rearranging terms, and observing that $\bE[\hat{\Pi}|\hat{W}] = \hat{W}$, we
have that 
\[
\bE~ \finner{\hat{\Pi}-\Pi^\star ,-C}  = \finner{\bE \,[\hat{W}]-\Pi^\star ,-C}  \le
\sqrt{\frac{8k^2(d+r)    \log (d)}{rm}} ~.
\]
The right-hand side of the above is smaller than $\epsilon$ if
$m \ge 8k \log(d)  \cdot \frac{d+r}{r} \cdot \frac{k}{\epsilon^2}$. Note, however, that we
also require that $m$ is large enough so that $\alpha \le
1/2$. Since $\alpha = \eta d^2$, it follows that $m$ should
also satisfy $m \ge \frac{2d^4 r \log(d)}{d+r}$. 
\end{proof}

\subsection{Comparison between the bounds in the passive and the
 active models}  \label{sec:comparison}
The left term in the bound of MBEG is smaller than the bound of POPCA if $\frac{d^2}{r^2} > 8 k \log (d) \cdot \frac{d+r}{r}$. 
The right term in the bound of MBEG is smaller than the bound of POPCA if 
$$\epsilon < \sqrt{\frac{(d+r)k}{2d^2 r^3\log(d)}}~.$$ 
Hence, if both of the conditions hold, then the
bound of MBEG is better than the bound of
POPCA. Also, if $\epsilon < \frac{2(d+r)k}{d^2r}$, then the
dependence of the sample complexity of MBEG on $d/r$ is linear. 

A reasonable regime in which MBEG
enjoys a linear price is when $r$ and $k$ are constants, and
$\epsilon$ is proportional to $1/d$. Note that in this regime, the
bound of POPCA scales with $d^4$, while the bound of MBEG scales
only with $d^3$ (in the full-information setting, the sample
complexity for this case scales with $d^2$). To summarize, ignoring the dependence
on $k$ (and logarithmic factors), MBEG attains the desired linear
price when $\epsilon$ is `small'. 

\section{Discussion}  \label{sec:discussion}
We introduced the problem of subspace learning with partial information, and considered both a passive and active model. Our first observation was that looking at a single coordinate does not give any information. Therefore, our algorithms look at the products of attribute pairs. Using the POPCA algorithm for the passive model, we showed that the sample complexity is tightly characterized by $\Theta( (d/r)^2 k / \epsilon^2 )$. Hence, the price of partial information in this case is quadratic.
For the active model we introduced the MBEG algorithm which exploits the gathered information in order to make a better choice over which attributes to observe. We showed that if the desired accuracy $\epsilon$ is small, then MBEG achieves a linear price. Since the expected number of observed attributes from each vector is $r$, we can not hope for a better dependence on $d/r$. At this point, a natural question arises: can we attain a linear price of partial information in the active model, independently of the required accuracy? We conclude with an observation regarding MBEG that provides a partial answer to this question.

Assume that $\epsilon$, $r$ and $k$ are constants. Examining the implementation of MBEG, we note that as long as the products, $x_s x_q$, between two consecutive observed attributes is zero, MBEG does not change the sampling distribution (which is initially uniform) nor its current estimation of the covariance matrix. Fix some attribute $j$ and consider the distribution $P_j$ which is concentrated on $e_j$. The probability that the product between two consecutive observed attributes is not zero is $(1/d)^2$. Since $r$ is constant, only zero products are observed for at least $\Omega(d^2)$ iterations, implying the same bound on the sample complexity.

The implication of this result is that in order to achieve a linear price for any value of  $\epsilon$, it is necessary to extract more information from the partial observations (e.g., take into account both the products of attribute pairs and the single attributes). Developing such algorithms or alternatively, tightening the lower bounds, is left for future research.

\acks{We thank the anonymous reviewers for their helpful comments. We also thank Amit Daniely for helpful discussions. This research has been supported by ISF no 1673/14.}


\newpage

\appendix
\section{Proof of \thmref{thm:lowerFull}}  \label{app:lowerFull}
In this section we prove \thmref{thm:lowerFull}.
\subsection{The adversarial distribution }
Let $\cU = \{u_j\}_{j=1}^{2k} \subseteq \reals^d$ be a set of $2k$ orthonormal vectors. In the proof sketch of \thmref{thm:lowerFull} we described a family $\cF$ of
distributions over the set $\{u,-u: u \in \cU\}$. Our lower bounds will be proved to hold in expectation when choosing
$P \in \cF$ at random. According to Yao's minimax principle, such a result implies that the lower bound
holds for some distribution in $\cF$.  

\subsection{A Successful Subspace Learner Is Also a Successful Coin identifier}
In this part we formalize the reduction from coin identification to subspace learning. As we sketched before, a key ingredient of our analysis of the lower bounds is the relation
between subspace learning and ``coin identification''.  This relation is formalized next. We
first need to introduce some additional notation. To specify the
distribution $P \in \cF$, fix a vector $(b_1,\ldots, b_k) \in
\{-1,1\}^k$. Let $\hat{\Pi} = \sum_{i=1}^k \hat{u}_i
\hat{u}_i^\top \in \proj$ (where $\hat{u}_1, \ldots, \hat{u}_k$ are
orthonormal). For every $(i,j) \in [k] \times [2k]$, define
$\theta_{i,j}=|\inner{\hat{u}_i, u_j}|$ to be the
\emph{covariance} between $\hat{u}_i$ and $u_j$. Next, for each $j \in [2k]$, define
$\theta_j^2 = \sum_{i=1}^k \theta_{i,j} ^2$. Also, define the set 
\begin{equation} \label{eq:indIdentifiedCoins}
J=\{j \in [k]: b_j=1 \wedge  \theta_j^2 > \theta_{j+k}^2 \} \cup \{j \in [k]: b_j=-1 \wedge  \theta_j^2 < \theta_{j+k}^2 \}~.
\end{equation}
For reasons that will become apparent shortly, we name $J$ the set of
identified coins. The following lemma asserts that a successful subspace learner must identify most of the coins.
\begin{lemma}  \label{lem:subspaceToCoins}
Let $\epsilon \in
(0,1)$. Assume that $L(\hat{\Pi}) - \min_{\Pi \in \proj} L(\Pi) \le  \epsilon$. Let
$J$ be defined as in \ref{eq:indIdentifiedCoins}.
Then, $|J|/k > 1- \frac{2 \epsilon}{\alpha}$.
\end{lemma}
\begin{proof}
Assume w.l.o.g. that $b_1=\ldots=b_k=1$. Note that
\[
\bE[x x^\top] = \frac{1}{k} \sum_{j=1} ^k \left(\frac{1+\alpha}{2} u_j
  u_j^\top+\frac{1-\alpha}{2} u_{j+k} u_{j+k} ^\top \right)~.
\]
The optimal projection is obtained by picking the largest eigenvectors
of $\bE[x x^\top]$. Precisely, the optimal projection matrix is
$\Pi^\star = \sum_{i=1} ^k u_i  u_i ^\top$. We assume the loss function as
formulated in \ref{eq:loss}. The loss
of $\hat{\Pi}$ is calculated as follows:
\begin{align}  \label{eq:theta}
\bE [\|x\|_2^2] -L (\hat{\Pi})&=\finner{\sum_{i=1}
^k \hat{u}_i \hat{u}_i ^\top, \bE[x x^\top] } \notag \\& 
= \frac{1}{k} \sum_{i=1} ^k \sum_{j=1} ^k
\left( \frac{1+\alpha}{2}\finner{\hat{u}_i \hat{u}_i ^\top,u_j u_j
  ^\top}+ \frac{1-\alpha}{2}\finner{\hat{u}_i \hat{u}_i ^\top,u_{j+k}
  u_{j+k} ^\top} \right) \notag \\&
= \frac{1}{k} \sum_{i=1} ^k \sum_{j=1} ^k
\left(\frac{(1+\alpha) \theta_{i,j}^2}{2} +\frac{(1-\alpha)
  \theta_{i,j+k}^2}{2} \right) \notag \\&
=  \frac{1}{k}  \sum_{j=1} ^k
\left(\frac{(1+\alpha) \theta_j^2}{2} +\frac{(1-\alpha)
   \theta_{j+k}^2}{2} \right) \notag \\&
= \frac{1}{2k}  \sum_{j=1} ^{2k} \theta_j^2
+\frac{\alpha }{2k}  \sum_{j=1} ^k (\theta_j^2-\theta_{j+k}^2)~.
\end{align}
Since $\{u_j\}_{j=1} ^{2k}$ is orthonormal, for each $i \in [k]$ we
have $\sum_{j=1} ^{2k}
\theta_{i,j}^2 \leq 1$. Hence, $\sum_{j=1} ^{2k} \theta_j^2 \le k$, so that the second-to-last term of \ref{eq:theta} is at
most $\frac{1}{2}$. This value is attained by $\Pi^\star$,
and for simplicity, we will assume that is attained by
$\hat{\Pi}$ as well.  For the same reasons, for each $i \in [k]$,
we have $\sum_{j=1} ^k  (\theta_{i,j}^2-\theta_{i,j+k}^2) \leq 1$. Hence,
$\sum_{j=1} ^{k} (\theta_j^2 -\theta_{j+k}^2) \le k$, so that the last term 
of \ref{eq:theta} is at most $\alpha /2$. Once again, this
value is attained by $\Pi^\star$. Our
last observation is that if the $j$-th coin is not identified, i.e.,
$j \notin J$, then $\theta_j^2-\theta_{j+k}^2
\leq 0$. Combining these observations, we obtain that
\[
L(\hat{\Pi})-L(\Pi^\star) \geq  \alpha  (1-|J|/k)/2~.
\]
The proof is completed by combining the assumption that
$L(\hat{\Pi})-L(\Pi^\star) \le \epsilon$.
\end{proof}

\subsubsection{Lower Bound on Coin Identification}
We previously informally argued that $\Omega(1/\alpha^2)$ samples are needed to
identify a coin with bias $\frac{1 \pm \alpha}{2}$. If we have $k$
such independent coins, then $\Omega(k/\alpha^2)$ are needed. Let us
formalize this result.

A coin identification problem with parameter $\alpha$ is defined as follows.
Consider a binary classification problem with a domain $[k]$ and label
set $\{0,1\}$. The
hypothesis class is the set $\cH=\{0,1\}^{[k]}$. The underlying distribution over $[k] \times \{0,1\}$ is chosen at
random using the following mechanism. The marginal distribution over $[k]$ is uniform,
and the conditional probability over the label (coin) is determined by
$P(y=1|x=j)=\frac{1+b_j \alpha}{2}$, where each $b_j$ is an independent
Rademacher random variable (drawn in advance). We observe that this distribution is identical to the distribution defined (over a shattered set
of size $k$) in \cite[Theorem
5.2]{anthony2009neural}. 

As usual, an algorithm for coin identification obtains an i.i.d. training sequence $S$
according to $P$, and has to return a hypothesis $h \in \cH$. The generalization error of $h \in \cH$, denoted $\textrm{err}(h)$, is defined to be the probability that it
misclassifies a new generated point $(x,y)$, i.e.,
$\err(h)=P(h(x) \neq y)$. The next theorem follows from \cite[Theorem
5.2]{anthony2009neural}.
\begin{theorem}  \label{thm:coins}
Let $\cB$ be an algorithm for coin identification, and let
$\tilde{\epsilon} \in (0,1/64)$. Consider a coin identification problem with
$\alpha=8 \tilde{\epsilon}$. If $m<\frac{k}{8 \tilde{\epsilon}^2}$, then there exists
a distribution $P$ for which
$$
\bE_{S \sim P^m} \,\err(\cB(S))-\min_{h \in \cH} \err(h)> \tilde{\epsilon}~.
$$ 
\end{theorem}

\subsection{Concluding the Theorem}
We are now in position to complete the reduction from the task of coin
identification to the task of subspace learning, and consequently conclude
\thmref{thm:lowerFull}. 

Let $\cA$ be a subspace learner whose sample complexity is
$m(d,k,r=d,\epsilon)$. We describe an algorithm $\cB$ 
for coin identification (with parameter $\alpha$) which uses $\cA$
as a subroutine. To this end, we shall provide a (randomized) map between the input of
$\cB$ to the input of $\cA$. These inputs have the form of training
sequences, which will be denoted by $S_{\mathrm{subspace}}$ and
$S_{\mathrm{coins}}$, respectively.  For each $j \in [k]$, the
pair $(j,1)$ is associated with $u_j$ or $-u_j$ with equal
probability. The pair $(j,0)$
is associated with $u_{j+k}$ of $-u_{j+k}$ with equal probability. Clearly, the sequence provided
to $\cA$ is generated according to the distribution discussed in
the proof sketch of \thmref{thm:lowerFull}. Given an accuracy parameter
$\tilde{\epsilon} \in (0,1/64)$ for $\cB$, we will require accuracy $\epsilon = \tilde{\epsilon}/2$ from $\cA$. To complete the
reduction, we will specify how the output of $\cB$ is
determined using the output of $\cA$. 

For
each $j \in [k]$, the output hypothesis returned by $\cB$ is defined
by
\[
h(j)=\begin{cases}  1 & \theta_j > 
\theta_{j+k} \\ 0 & \textrm{otherwise}\end{cases}~.
\]
Denote the set of identified coins by
$J$ (as in \lemref{lem:subspaceToCoins}). Observe that each coin that
is not identified, adds $\alpha/k$ to the
relative error of $\cB$. It follows from \lemref{lem:subspaceToCoins} that 
\begin{align}  \label{eq:almostDoneCoins}
\textrm{err}(\cB(S_{\mathrm{coins}}))-\min_{h \in \cH} \err(h)&=\alpha
(1-|J|/k) \notag\\&
 \leq  2 \left(L(\cA(S_{\mathrm{subspace}}))-\min_{\Pi \in \proj} L(\Pi) \right)~.
\end{align}
If $m \ge m(d,k,r=d,\epsilon)$,
then the right-hand side of \ref{eq:almostDoneCoins} is at most $2\epsilon=\tilde{\epsilon}$.
The proof is completed by applying \thmref{thm:coins} (with
$\alpha = 8\tilde{\epsilon}$).

\section{Extending \lemref{lem:varMBEG} to $\mathbf{r>2}$}  \label{app:mbegRg2}
We will now extend the proof of \lemref{lem:varMBEG} to any (even)
$r>2$. Fix an iteration $i$, and denote $\hat{C} = \hat{C}_i, W = W_i$.
We have
\begin{align*}
\bE_i \inner{W,\hat{C}^2} &= \frac{4}{r^2} \left [\sum_{j=1} ^{r/2} \bE_i
\inner{W,\hat{C}_j^2} + \sum_{j \neq t} \bE_i \inner{W,\hat{C}_j
 \hat{C}_t} \right]~.
\end{align*}
From the case $r=2$, we already know that the term $\bE_i \inner{W,\hat{C}_j^2}$ is at most $dk$. Fix some pair $j \neq t$. Note that $\hat{C}_j\hat{C}_t
\neq 0$ if and only if (at least) one of the indices $s_j,q_j$ is
equal to one of the indices $s_t, q_t$. Hence,
\begin{align*}
\bE_i \inner{W,\hat{C}_j \hat{C}_t} &\le 2 \sum_{(s,q) \in [d]^2}
p_{s,q}^2 (W_{s,s} + W_{q,q}) \frac{x_s^2 x_q^2}{4 p_{s,q}^2} \\&
+ 4\sum_{\substack{(s,q,q') \in [d]^3: \\ q \neq q'}}
p_{s,q} p_{s,q'} W_{q,q'} \frac{x_s^2 |x_q| |x_q'|}{4 p_{s,q} p_{s,q'}}~.
\end{align*}
The first term can be bounded as follows:
\begin{align*}
2\sum_{(s,q) \in [d]^2} p_{s,q}^2 (W_{s,s} + W_{q,q}) \frac{x_s^2
  x_q^2}{4 p_{s,q}^2} &= \frac{2}{4} \sum_{(s,q) \in [d]^2}  (W_{s,s} + W_{q,q}) x_s^2
  x_q^2\\& 
= \frac{2}{4} \left \langle {(W_{s,s} + W_{q,q})_{(s,q) \in [d]^2}, (x_s^2
  x_q^2)_{(s,q) \in [d]^2}} \right \rangle\\&
\le \frac{2}{4} \|(W_{s,s} + W_{q,q})_{(s,q) \in [d]^2} \|_\infty \|(x_s^2
  x_q^2)_{(s,q) \in [d]^2} \|_1 \\&
\le  \frac{2}{4} \cdot 2 
= 1~.
\end{align*}
The second term can be bounded as:
\begin{align*}
4 \sum_{\substack{(s,q,q') \in [d]^3: \\ q \neq q'}} p_{s,q} p_{s,q'} W_{q,q'} \frac{x_s^2 |x_q|
  |x_{q'}|}{4 p_{s,q} p_{s,q'}}  &
\le 4\cdot \frac{1}{4} \left (\sum_{s \in [d]} x_s^2 \right) \left(\sum_{(q,q') \in [d]^2} |W_{q,q'}| \cdot |x_q
  x_{q'}| \right) \\&
\le  \tnorm{W} \snorm{xx^\top} \\&
\le   k ~.
\end{align*}
Combining the above, we obtain 
\begin{align*}
\bE_i \inner{W,\hat{C}^2} &= \frac{4}{r^2} \left [\frac{r}{2} d k +
  \frac{r}{2} \left(\frac{r}{2}-1 \right) (k+1) \right] \\&
\le \frac{2dk}{r} +k +1  \\&
\le \frac{2dk}{r} +2k \\&
= \frac{2k(d+r)}{r}~.
\end{align*}

\section{The Price of Partial Information is at Least Linear}  \label{app:linearLower}
\begin{theorem}
Let $k=1$. Then, the sample complexity of subspace learning is bounded
below by:
\[
m(d,k=1,r,\epsilon) \ge \Omega \left( \frac{d}{r \epsilon} \right)~.
\]
\end{theorem}

\begin{proof}
Let $\epsilon  \in (0,1/4)$, and let $\cA$ be a bandit subspace learner. For
each $s \in [d]$, we define a distribution $P_s$ as follows. The
zero vector is drawn with probability $1-c\epsilon$ (where $c>2$), and the
vector $e_s$ is drawn with probability $c\epsilon$.
Note that $\bE [xx^\top]=c \epsilon \, e_s e_s^\top$, and thus the
optimal projection is given by $\Pi^\star:=e_s e_s^\top$. 

We next show that  a successful
subspace learner must identify the distribution. Let $P_s$ be a concrete distribution. Denote by $\hat{\Pi} = \hat{u}
\hat{u}^\top$ the output of the learner. Define $\theta_s =
u_s^2$. It can be easily seen that if $L(\hat{\Pi})-L(\Pi^\star)<
\epsilon < c\epsilon/2$, then $\theta_s >
\theta_q$ for any $q \neq s$. That is, a successful subspace learner
must identify the index $s$. 

Next we observe that if the size of the sample is at most $o \left
(\frac{d}{r \epsilon} \right)$, then with a
non-negligible probability, all the observations made by the learner are equal to zero. It follows that there
exists a distribution $P_j$ for which, $\bE [ L(\hat{\Pi})-L(\Pi^\star)] > c \epsilon/2 > \epsilon$.
\end{proof}

\section{Decomposing Elements in $\projConv$ into a Convex
  Combination of Elements From $\proj$}  \label{app:decomposition}
In this part we detail the decomposition procedure mentioned in
\lemref{lem:convChar}. For a symmetric $d \times d$ matrix $A$, we denote its eigenvalues by $\lambda(A)_1 \geq \ldots \geq \lambda(A)_d$. For convenience, we define two subroutines. The procedure $\mathrm{sort}(x,s)$ returns a
set of $s$ indices, corresponding to the $s$ largest values of
$x$. Given a vector $x \in \reals^d$, the function $\diag:\reals^d \rightarrow
\reals^{d \times d}$ returns a
diagonal matrix $X$ with $X_{j,j} = x_j$.

\begin{algorithm}
\caption{Decomposition Procedure}  \label{alg:decomposition}
\begin{algorithmic}
\STATE \textbf{Input: } $W \in \projConv = \conv(\proj)$
\STATE perform eigendecomposition $W = U \diag(\lambda(A)) U^\top$
\STATE $\lambda := k^{-1} (\lambda(A)_1, \ldots, \lambda(A)_d)$ 
\STATE $i:=1$
\REPEAT
\STATE $J := \sort(\lambda,k)$
\STATE $c_i := k^{-1} \sum_{j \in J} e_j$
\STATE $s := \min_{j \in J} \lambda_j$
\STATE $l := \max_{j \in [d] \setminus J} \lambda_j$
\STATE $\beta_i := \min \left \{sk, \sum_{j=1} ^d
    \lambda_j - lk \right\}$
\STATE $\lambda := \lambda  - \beta_i c_i$
\STATE $\Pi_i := U \diag (kc_i) U^\top$
\STATE $i := i+1$
\UNTIL $\lambda=(0,\ldots,0)$
\STATE For each $j \in [i-1]$ choose $A = \Pi_j$ with probability $\beta_j$
\end{algorithmic}
\end{algorithm}

It has been proved by \cite{warmuth2008randomized} that the loop
inside \algref{alg:decomposition} repeats at most $d$ times, and
decomposes $W$ into a convex combination $\sum \beta_i \Pi_i$ of
elements from $\proj$.

\section{Covariance Estimation with Missing Entries} \label{app:lounici}
Corollary 1 in \cite{lounici2014high} provides bounds under the assumption that the instances are drawn from a sub-gaussian distribution (see assumption 1 in \cite{lounici2014high}). This assumption is weaker than our boundedness assumption. Translating the results to our notation yields:
\begin{lemma}
Given $m$ partial observations, let $\hat{C}$ be the unbiased estimation constructed by POPCA. Let $\lambda>0$ be a parameter.
\[
\tilde{C}: = \tilde{C(\lambda)} = \argmin_{S \in \symps{d}} \{\|S-\hat{C}\|_F^2+\lambda \|S\|_1\}~,
\]
where $\symps{d}$ is the set of symmetric positive semi-definite $d \times d$ matrices. Then, for a suitable choice of $\lambda$, with probability at least $1-1/(2d)$, we have
\[
\|\tilde{C} - C\|_F \le \sqrt{\inf_{S \in \symps{d}} \{\|S-C\|_F^2+c_1 \lambda^2_1(C) \cdot \frac{\tr(C)}{\lambda_1(C)} \cdot \frac{d^2}{r^2 m} \cdot \textrm{rank}(S) \cdot \log(2d)\}}~,
\]
where $c_1$ is a constant and $\lambda_1(C)$ is the leading eigenvalue of $C$. 
\end{lemma}
Substituting $m=\left \lceil (d/r)^2 \cdot \frac{k}{\epsilon^2}\right \rceil$ from \lemref{lem:covarianceConverge} into the above bound, we obtain that the RHS is at most 
\[
\sqrt{\inf_{S \in \symps{d}} \{\|S-C\|_F^2+ c_1 \lambda^2_1(C) \cdot \frac{\tr(C)}{\lambda_1(C)} \cdot \frac{\epsilon^2}{k} \cdot \textrm{rank}(S) \cdot \log(2d))\}}~.
\]
This bound is always worse than our bound (i.e., larger than $\epsilon/\sqrt{k}$).
\vskip 0.2in
\bibliography{bib}

\end{document}